\documentclass[twoside]{article}
\usepackage[accepted]{aistats2016}%_arxiv}

\usepackage{hyperref}
\usepackage{url}
\usepackage[numbers]{natbib}
\usepackage{algorithmic,algorithm}
\usepackage{psfrag,graphicx,graphics,epstopdf,colortbl}
\usepackage{fancyhdr}
\usepackage{amsmath,amsfonts}
\usepackage{amsthm,amssymb,amsopn}
\usepackage{wrapfig}
\usepackage{afterpage}
\usepackage{textcomp}
\usepackage{amsmath}
\usepackage{footnote}
\makesavenoteenv{tabular}
\makesavenoteenv{table}

\newcommand\independent{\protect\mathpalette{\protect\independenT}{\perp}}
\def\independenT#1#2{\mathrel{\rlap{$#1#2$}\mkern2mu{#1#2}}}
\newcommand{\Prob}{\Pr}
\DeclareMathOperator{\Lap}{Lap}
\newcommand{\aaa}{\mathbf{a}}

\newcommand{\x}{\mathbf{x}}

\newcommand{\y}{\mathbf{y}}

\newcommand{\w}{\mathbf{w}}

\newcommand{\rr}{\mathbf{r}}
\newcommand{\bb}{\mathbf{b}}

\newtheorem{define}{Definition}
\newtheorem{thm}{Theorem}

\newcommand{\Vc}{\mathcal{D}}

\newcommand{\Xc}{\mathcal{X}}

\newcommand{\Ac}{\mathcal{A}}

\newcommand{\ybt}{\tilde{\y}}
\newcommand{\xbt}{\tilde{\x}}

\newcommand{\rbt}{\tilde{\rr}}

\newcommand{\Xb}{\mathbf{X}}

\newcommand{\Dc}{\mathcal{D}}

\newcommand{\argmin}{\operatornamewithlimits{argmin}}

\newcommand{\kqwshorten}[1]{}

\newcommand{\eat}[1]{}
\begin{document}



\twocolumn[

\aistatstitle{Private Causal Inference}

\aistatsauthor{ Matt J. Kusner \And Yu Sun \And Karthik Sridharan \And Kilian Q. Weinberger }

\aistatsaddress{ Washington University in St. Louis \\ \texttt{mkusner@wustl.edu}  \And Cornell University \\ \texttt{ys646@cornell.edu} \And Cornell University \\ \texttt{sridharan@cs.cornell.edu} \And Cornell University \\ \texttt{kqw4@cornell.edu}} 
%\aistatsaddress{ \texttt{mkusner@wustl.edu} \And \texttt{ys646@cornell.edu} \And \texttt{sridharan@cs.cornell.edu} \And \texttt{kqw4@cornell.edu} } 

]

%\input{abstract.tex}
% \begin{abstract}
%   The Abstract paragraph should be indented 0.25 inch (1.5 picas) on
%   both left and right-hand margins. Use 10~point type, with a vertical
%   spacing of 11~points. The {\bf Abstract} heading must be centered,
%   bold, and in point size 12. Two line spaces precede the
%   Abstract. The Abstract must be limited to one paragraph.
% \end{abstract}

%!TEX root=pci.tex
\begin{abstract} 
%Causal inference is now a well-studied discipline. Different from methods aimed just towards classification and regression, 
%the discovery of underlying causal structures enables prediction of how manipulating random variables (the `causes') affect other random variables (the `effects'). 
%Causal inference infers conclusions about the causal connection between random variables. Recent research has resulted in practical algorithms with high  impact potential, especially in medical research and treatment discovery. However, many of the promising applications for causal inference include sensitive personal data that should be kept private (i.e., medical records). Therefore, there is an immediate need for the development of causal inference methods that preserve data privacy. We address this current shortcoming using one of the most popular causal inference frameworks: the additive noise model (ANM) \cite{hoyer2009nonlinear}. We derive a framework that provides strong privacy guarantees for a variety of ANM variants and show in extensive real-world experiments that our techniques are practical, easy to implement, and have great promise for a variety of privacy settings. 
% in medical fields have the potential to influence disease treatment
Causal inference deals with identifying which random variables ``cause" or control other random variables. Recent advances on the topic of causal inference based on tools from statistical estimation and machine learning have resulted in practical algorithms for causal inference. Causal inference has the potential to have significant impact on medical research, prevention and control of diseases, and identifying factors that impact economic changes to name just a few. However, these promising applications for causal inference are often ones that involve sensitive or personal data of users that need to be kept private (e.g., medical records, personal finances, etc). Therefore, there is a need for the development of causal inference methods that preserve data privacy. We study the problem of inferring causality using the current, popular causal inference framework, the additive noise model (ANM) while simultaneously ensuring privacy of the users. Our framework provides differential privacy guarantees for a variety of ANM variants. We run extensive experiments, and demonstrate that our techniques are practical and easy to implement. 
%The experiments show promise for a variety of causal inference scores studied in literature. 
\end{abstract}

\section{Introduction}
%!TEX root=pci.tex

% causal inference is mature
Causal identification allows one to reason about how manipulations of certain random variables (the causes) affect the outcomes of others (the effects). 
Uncovering these causal structures has implications ranging from creating government policies to informing health-care practices. 
Causal inference was motivated by the impossibility of randomized intervention experiments in many cases, and the ambiguity of conditional independence testing \cite{spirtes2000causation,pearl2000causality}.
In the absence of interventions, it attempts to discover the underlying causal relationships of a set of random variables entirely based on samples from their joint distribution. 
The field of causal inference is now a mature research area, covering learning topics as diverse as supervised batch inference \cite{lopez2015towards,mooij2014distinguishing,peters2014causal}, time-series causal prediction \cite{geiger2015causal}, and linear dynamical systems \cite{shajarisales2015telling}. 
Many inference methods require only a regression technique and a way to compute the independence between two distributions given samples \cite{hoyer2009nonlinear,janzing2012information}. 

One would hope that researchers could publicly release their causal inference findings to inform individuals and policy makers. One of the primary roadblocks to doing so is that often causal inference is performed on data that individuals may wish to keep private, such as data in the fields of medical diagnosis, fraud detection, and risk analysis. 
%For instance, an individual may not wish to reveal that he/she has cancer, and therefore be hesitant  to participate in a study that investigates the cause of the disease. 
Currently, no causal inference method has formal guarantees about the privacy of individual data, which may be able to be inferred via attacks such as reconstruction attacks \cite{dinur2003revealing}.

% differential privacy
Arguably one of the best notion of privacy is differential privacy, introduced by \citet{dwork2006calibrating} and since used throughout machine learning \cite{dwork2009differential,jain2011differentially,mcsherry2007mechanism,chaudhuri2011differentially,dwork2013algorithmic}. Differential privacy guarantees that the outcome of an algorithm only reveals aggregate information about the entire dataset and never about the individual. An individual who is considering to participate in a study can be reassured that his/her personal information cannot be recovered with extremely high probability. 
%This can be reassuring to an individual who is considering providing their information to a dataset, that provably no one can recover their private values.

% in this paper we do private causal inference
To our knowledge, this paper is the first to investigate private causal inference. We show that it is possible to privately release the quantities produced by the highly-successful additive noise model (ANM) framework by adding small amounts of noise, as dictated by differential privacy. Furthermore, these private quantities, with high probability, do not change the causal inference result, so long as it is confident enough. We demonstrate on a set of real-world causal inference datasets how our privacy-preserving methods can be readily and usefully applied.

\section{Related Work}
%!TEX root=pci.tex

% % in this paper we solve this problem
% In this work, as far as we know, we derive the first techniques towards private causal inference. We show that for certain dependence scores in the ANM causal framework, we can release them privately, so that if the dataset used to compute this score changes, it is undetectable. To do so, we use the framework of Differential Privacy \cite{dwork2006calibrating}, which at a high level modifies the dependence scores by some amount of noise to mask their true values. We show that these private (noisy) scores very likely do not change the direction of causal inference, only doing so if the original scores are close, indicating that the direction is in any case uncertain. Finally, we can also preserve privacy if instead the data used to fit the regression functions from $X$ to $Y$ and $Y$ to $X$ changes, and again show that the causal direction will very likely stay the same. We demonstrate on a set of real-world causal inference datasets how our privacy-preserving methods can be readily and usefully applied.

% causality is an important field
Discovering the causal nature between random events has captivated researchers and philosophers long before the formal developments of statistics. This interest was formalized by \citet{reichenbach1991direction} who argued that \emph{all} statistical correlations in data arise from underlying causal structures between the concerned random variables. For example, the correlation between smoking and lung cancer was found to arise from a direct causal link \cite{centers2010tobacco}. 

% casual inference is particularly important
%Different from methods aimed just towards classification and regression, accurate causal identification allows one to reason about how manipulations of certain random variables (the causes) affect the outcomes of others (the effects). Therefore, uncovering these causal structures has implications ranging from creating government policies to informing health-care practices.

% the gold standard in causal identification and a problem
%The gold standard in causal identification is to perform a randomized intervention experiment in which the researcher randomly fixes the outcomes of a random variable $X$ and observes how this affects the outcomes of another random variable $Y$. If the randomly-chosen outcomes of $X$ are correlated with the outcomes of $Y$ then there exists a causal relationship from $X$ to $Y$, written $X \rightarrow Y$. While such interventions are conceptually simple, they are in many cases cost-impractical, technically-impossible, or even more seriously morally-questionable. As an extreme example, implementing an intervention to answer whether smoking $X$ causes cancer $Y$ would require making individuals smoke different amounts for a period of time and observing their cancer outcomes to determine if $X \rightarrow Y$. Moreover, if one wanted to identify if $Y \rightarrow X$ one would have to induce different cancer outcomes and observe smoking outcomes. Therefore, there has been a wealth of research towards determining causal structure without having to resort to interventions.

% one popular recent CI work is ANMs
One of the most popular causal inference alternatives to conditional independence testing is the Additive Noise Model (ANM) approach developed by \citet{hoyer2009nonlinear} and used in many recent works \cite{zhang2009identifiability,stegle2010probabilistic,kpotufe2014consistency,buhlmann2014cam}. ANMs, originally designed for inferring whether $X \rightarrow Y$ or $Y \rightarrow X$ and later extended to large numbers of random variables, work under the assumption that the effect is a non-linear function of the cause plus independent noise. ANMs are one of many proposed causal inference methods in recent literature \cite{janzing2012information,geiger2014estimating,lopez2015towards,sgouritsa2015inference}

% causal inference%
%One promising research direction is the field of \emph{causal inference}. The goal of causal inference is, given a set of samples from the joint distribution over the random variables of interest, to infer the causal structure between the variables. 
Work by \citet{spirtes2000causation,pearl2000causality} shows how to determine if $X\rightarrow Y$ when these variables are a part of a larger `causal network', via conditional independence testing. One downside to conditional independence based approaches is that inherently they cannot distinguish between Markov-equivalent graphs. Thus it may be possible that a certain set of conditional independences imply both $X \rightarrow Y$ and $Y \rightarrow X$. Furthermore, if $X$ and $Y$ are the only variables in the causal network there is no conditional independence test to determine whether $X \rightarrow Y$ or $Y \rightarrow X$.

\section{Background}
\label{sec:background}
%!TEX root=pci.tex
% additive noise models def
% ANM algorithm
% dependence scores
% diff privacy
Our aim is to protect the privacy of individuals who submit personal information about two random variables of interest $X$ and $Y$. Their information should remain private when it is used to infer whether $X$ causes $Y$ ($X\!\rightarrow\!Y$), or $Y$ causes $X$ ($Y \rightarrow X$) using the ANM framework. This personal information comes in the form of i.i.d. samples $\{(x_i,y_i)\}_{i=1}^n$ from the joint distribution $\mathbb{P}_{X,Y}$. We will assume that, 1. There is no confounding variable $Z$ that commonly causes or is a common effect of $X$ and $Y$. 2. $X$ and $Y$ do not simultaneously cause each other.

%Given two random variables $X,Y$ our goal is to infer whether $X$ causes $Y$ ($X\!\rightarrow\!Y$), or $Y$ causes $X$ ($Y \rightarrow X$), purely from i.i.d. samples $\{(x_i,y_i)\}_{i=1}^n$ from their joint distribution $\mathbb{P}_{X,Y}$. We will assume that, 1. There is no confounding variable $Z$ that commonly causes or is a common effect of $X$ and $Y$. 2. $X$ and $Y$ do not simultaneously cause each other.

\subsection{Additive Noise Model}
Deciding on the causal direction between two variables $X$ and $Y$ from a finite sample set has motivated an array of research \cite{friedman2000gaussian,kano2003causal,sun2008causal,hoyer2009nonlinear,zhang2009identifiability,mooij2011causal,janzing2012information,kpotufe2014consistency,lopez2015towards}. Perhaps one of the most popular results is the Additive Noise Model (ANM) proposed by \citet{hoyer2009nonlinear}. The ANM framework assumption is defined as follows.
%
%Even with the above assumptions it is non-trivial to infer the direction of the causal influence from a finite sample set. A number of models have been proposed towards this inference \cite{}. One of the primary models is the Additive Noise Model (ANM) proposed by \citet{hoyer2009nonlinear}.
\begin{define}
Two random variables $X,Y$ with joint density $p(x,y)$ are said to `satisfy an ANM' $X\!\rightarrow\!Y$ if there exists a non-linear function $f: \mathcal{R} \!\rightarrow\! \mathcal{R}$ and a random noise variable $N_Y$, independent from $X$, i.e. $X \independent N_Y$, such that
%Given random variables $X,N_Y$ with densities $p_X, p_{N_Y}$ and a \emph{non-linear} function $f: \mathcal{R} \rightarrow \mathcal{R}$, an additive noise model (ANM) of the form $X\!\rightarrow\!Y$ is defined as,
\begin{align}
Y = f(X) + N_Y. \nonumber
\end{align}
%The joint distribution $\mathbb{P}_{X,Y}$ with density $p(x,y)$ is said to `satisfy an ANM' $X\!\rightarrow\!Y$. 
\end{define}
As defined, an ANM $X\! \rightarrow\! Y$ implies a functional relationship mapping $X$ to $Y$, alongside independent noise. In order for this model to be useful for causal inference we would like the induced joint distribution $\mathbb{P}_{X,Y}$ for this ANM to be somehow identifiably different from the one induced by the ANM $Y \rightarrow X$. If so, we say that the causal direction is \emph{identifiable} \cite{mooij2014distinguishing}. If not, we have no hope of recovering the causal direction purely from samples under the ANM.

\citet{hoyer2009nonlinear} showed that ANMs are generically identifiable from i.i.d. samples from $\mathbb{P}_{X,Y}$ (except for a few special cases of non-linear functions $f$ and noise distributions). The intuition behind this is for the $X\!\rightarrow\!Y$ ANM, consider for most non-linear $f$ and (for simplicity) $0$-mean $N_Y$, the density $p(y|x)$ has mean $f(x)$ with distribution given by $N_Y$. %However, $p(x|y) = f^{-1}(y - N_Y)$ (if $f^{-1}$ even exists) does not necessarily have mean $f^{-1}(y)$
This implies that $p(y-f(x)|x)$ has distribution $N_Y$ that is independent of $X$. However, $p(x-f^{-1}(y)|y)$ is for many choices of $f$ and $N_Y$ not independent of $y$.

\begin{algorithm}[H]                      % enter the algorithm environment
\caption{ ANM Causal Inference \cite{mooij2014distinguishing}  }          % give the algorithm a caption
\label{alg:anm_ci}                           % and a label for \ref{} commands later in the document
\begin{algorithmic}[1]                    % enter the algorithmic environment
%	\STATE \textbf{Input:} $\Vc$; $\Lambda \subseteq \mathbb{R}^d$; $T$; $(\epsilon,\delta)$; $\sigma^2_{\Vc,0}$; $\gamma_T$
%	\STATE $\mu_{\Vc,0} = 0$
	\STATE \textbf{Input:} train/test data $\{x_i,y_i\}_{i=1}^n$, $\{x_i',y_i'\}_{i=1}^{m}$
	%\STATE \textbf{      } \;\;\;\;\;\;\;\;\; test data $\{x_i',y_i'\}_{i=1}^{m}$
	\STATE Regress on training data, to yield $\hat{f},\hat{g}$, such that:
	\STATE $\hat{f}(x_i) \approx y_i$, \;\; $\hat{g}(y_i) \approx x_i,\;\; \forall i$%\; \forall i$
	%\STATE $\hat{g}(y_i) \approx x_i,\; \forall i$
	\STATE Compute residuals on test data:
	\STATE $\rr_Y' := \y' - \hat{f}(\x')$, \; $\rr_X' := \x' - \hat{g}(\y')$
	%\STATE $\rr_X' := \x' - \hat{g}(\y')$
	\STATE Calculate dependence scores:
	\STATE $s_{X\!\rightarrow\!Y} := s(\x',\rr_Y')$, \;\; $s_{Y \rightarrow X} := s(\y',\rr_X')$
	%\STATE $s_{Y \rightarrow X} := s(\y',\rr_X')$
	\STATE \textbf{Return:} $s_{X\!\rightarrow\!Y}$, $s_{X\!\rightarrow\!Y}$, and $D$, where
	\STATE D = $\begin{cases}
				X\!\rightarrow\!Y & \mbox{ if } s_{X\!\rightarrow\!Y} < s_{Y \rightarrow X} \\
				Y\!\rightarrow\!X & \mbox{ if } s_{X\!\rightarrow\!Y} > s_{Y \rightarrow X} 
				%? & \mbox{ if } s_{X\!\rightarrow\!Y} = s_{Y \rightarrow X}
			\end{cases}$
\end{algorithmic}
\end{algorithm}

\subsection{Inferring Causality}
%\label{alg:anm_ci}
\citet{mooij2014distinguishing} give a practical algorithm for determining the causal relationship between $X$ and $Y$ (i.e., either $X\!\rightarrow\!Y$ or $Y \rightarrow X$), as shown in Algorithm~\ref{alg:anm_ci}. 
% which we reproduce in Algorithm~\ref{alg:anm_ci}.
The first step is to partition the i.i.d. samples into a training and a testing set. 
We use the training set to train the regression functions $\hat{f} : X\!\rightarrow\!Y$ and $\hat{g}: Y \rightarrow X$. 
We use the testing set to compute the residuals $\rr_Y' = \y' - \hat{f}(\x')$ and $\rr_X' := \x' - \hat{g}(\y')$. If we have an ANM $X\!\rightarrow\!Y$ then the residual $\rr_Y'$ is an estimate of the noise $N_Y$ which is assumed to be independent of $X$. Therefore, we calculate the dependence between the residual $\rr_Y'$ and the input $\x'$, $s_{X\!\rightarrow\!Y} := s(\x',\rr_Y')$, and $s_{Y \rightarrow X} := s(\y',\rr_X')$, using a dependence score $s(\cdot,\cdot)$.  If $s_{X\!\rightarrow\!Y}$ is less than $s_{Y\!\rightarrow\!X}$, 
%that found between the residuals of the ANM $Y \rightarrow X$ and $\y'$, 
then we declare $X\!\rightarrow\!Y$, otherwise $Y\!\rightarrow\!X$. 
%

%The work of \citet{kpotufe2014consistency} suggests that one achieves similar results from using this partitioning versus using the full sample to compute each quantity.

%Second, the algorithm trains the regression functions $\hat{f} : X\!\rightarrow\!Y$ and $\hat{g}: Y \rightarrow X$, and computes the residuals of these functions. As described above, if we have an ANM $X\!\rightarrow\!Y$ then the residual $\rr_Y' = \y' - \hat{f}(\x')$ is an estimate of the noise $N_Y$ (similarly for $Y \rightarrow X$) which is assumed to be independent of $X$. Therefore, we calculate the dependence between the residual $\rr_Y'$ and the input $\x'$. If this dependence is less than that found between the residuals of the ANM $Y \rightarrow X$ and $\y'$, then we declare $X\!\rightarrow\!Y$.

\subsection{Dependence Scores}

Crucially, the ANM approach hinges on the choice of dependence score $s(\cdot, \cdot)$. There have been many proposals, and we give a quick review of the most popular methods (for a detailed review see \citet{mooij2014distinguishing}).

\textbf{Spearman's $\rho$} is a rank correlation coefficient that describes the extent to which one random variable is a monotonic function of the other. Specifically, imagine independently sorting the observations $\{a_1,\ldots,a_m\}$ and $\{b_1,\ldots,b_m\}$ by value in increasing order. Let $o_{i}^a$ be the rank of $a_i$ in the $a$-ordering, and similarly, $o_i^b$ for $b_i$ in the $b$-ordering. Then Spearman's $\rho$ is,
\begin{align}
s(\aaa,\bb) := \Bigg| 1 - \frac{6 \sum_{i=1}^m d_i^2}{m(m^2 - 1)} \Bigg| \nonumber
\end{align}
where $d_i \!:=\! (o_i^a - o_i^b)$ are the rank differences for $\aaa,\bb$.

\textbf{Kendall's $\tau$.}
Similar to Spearman's $\rho$, the Kendall $\tau$ rank score calls a pair of indices $(i,j)$ \emph{concordant} if it is the case that $a_i\!>\!a_j$ and $b_i\!>\!b_j$. Otherwise $(i,j)$ is called \emph{discordant}. Then the dependence score is defined as
\begin{align}
s(\aaa,\bb) := \frac{|C - D|}{\frac{1}{2} m(m-1)} \nonumber
\end{align}
where $C$ is the number of concordant pairs and $D$ is the number of discordant pairs.

\textbf{HSIC Score.} The first proposed score for the ANM causal inference is based on the Hilbert-Schmidt Independence Criterion (HSIC) \cite{gretton2005measuring}, which was used by \citet{hoyer2009nonlinear}. They compute an estimate of the $p$-value of the HSIC under the null hypothesis of independence, selecting the causal direction having the lower $p$-value. 
% yielding the score:
%\begin{align}
%s(\aaa,\bb) := \hat{p}_{\mbox{HSIC}_{k_{\theta(\aaa)},k_{\theta(\bb)}}}(\aaa,\bb)
%\end{align}
Alternatively, one can use an estimator to the HSIC value itself:
\begin{align}
s(\aaa,\bb) := \widehat{\mbox{HSIC}}_{k_{\theta(\aaa)},k_{\theta(\bb)}}(\aaa,\bb)
\end{align}
where $k_\theta$ is a kernel with parameters $\theta$. \citet{mooij2014distinguishing} show that under certain assumptions the algorithm in section~\ref{alg:anm_ci} with the HSIC dependence score is consistent for estimating the causal direction in an ANM. 
%note we can also imagine a fixed kernel $k$.

% \textbf{Entropy Score.} The differential entropy of the data and residuals was used as a dependence score in \citet{kpotufe2014consistency}. They suggest the score function
% \begin{align}
% s(\aaa,\bb) := \hat{H}(\aaa) + \hat{H}(\bb), \label{eq:ent_score}
% \end{align}
% for inferring causality, where $\hat{H}$ is an estimator of the differential entropy.

\textbf{Variance Score.}
When the noise variables in the ANM are Gaussian, the variance score was proposed in \citet{buhlmann2014cam}, and defined as $s(\aaa,\bb) := \log \mathbb{V}(\aaa) + \log \mathbb{V}(\bb)$. Changes to a single input value can induce arbitrarily large changes to this score, which makes the variance score ill suited to preserve differential privacy. 

\textbf{IQR Score.}
We introduce a robust version of this score by replacing the variance of the random variables with their interquartile range (IQR). The IQR is the difference between the third and first quartiles of the distribution and can be estimated empirically. We defined the following IQR-based score:
\begin{align}
s(\aaa,\bb) := \log \mbox{IQR}(\aaa) + \log \mbox{IQR}(\bb). \label{eq:iqr_score}
\end{align} 
%The IQR also measures distribution spread, but for the purposes of privacy has nicer sensitivity properties \cite{dwork2009differential}. 

%\begin{align}
%s(\aaa,\bb) := \hat{H}(\aaa) + \hat{H}(\bb). \label{eq:ent_score}
%\end{align}
%where $\hat{H}$ is an estimator of the differential entropy. We consider a robust estimator of  based on 

% As the variance of a random variable upper-bounds its differential entropy \cite{cover2012elements}, and 
% as such \citet{buhlmann2014cam} consider the score: $s(\aaa,\bb) := \log \V[\aaa] + \log \V[\bb]$. We consider a slightly different score based on a more robust estimate of the spread of a random variable, the interquartile range (IQR):
% \begin{align}
% s(\aaa,\bb) := \log \mbox{IQR}(\aaa) + \log \mbox{IQR}(\bb) \label{eq:iqr_score}
% \end{align}

%\textbf{MML Score.}
%\cite{stegle2010probabilistic} describe a minimum message length (MML) score based on 

\subsection{Differential Privacy} % TODO: Need to fix this section
%In the following section we describe the privacy framework used in this work.

We assume that the data set $\Vc = \{(x_i,y_i)\}$ contains sensitive data that should not be inferred from the release of the dependence scores. 
One of the most widely accepted mechanisms for private data release is \emph{differential privacy} \citep{dwork2006calibrating}. In a nutshell it ensures that the released scores can only be used to infer aggregate information about the data set and never about an individual datum $(x_i,y_i)$. 

Let us define the Hamming distance between two data sets $d_H(\Dc,\tilde{\Dc})$ between two data sets  $\Dc$ and $\tilde{\Dc}$ as the number of elements in which these two sets differ. If a data set $\Dc$ is changed to $\tilde{\Dc}$, a distance $d_H(\Dc,\tilde{\Dc})\leq 1$ implies that at most one element was added, removed, or substituted. 

\begin{define}
\label{def:dp}
A randomized algorithm $\Ac$ is $(\epsilon,\delta)$-\textbf{differentially private} for $\epsilon,\delta \geq 0$ if for all $\mathcal{O} \!\in\! \emph{\mbox{Range}}(\Ac)$ and for all neighboring datasets $\Dc,\tilde{\Dc}$ with $d_H(\Dc,\tilde{\Dc})\leq 1$ we have that
\begin{align}
\Prob\bigl[\Ac(\Dc) = \mathcal{O}\bigr] \leq e^{\epsilon} \Prob\bigl[ \Ac(\tilde{\Dc}) = \mathcal{O}\bigr] + \delta. \label{eq:dp}
\end{align}
\end{define}
One of the most popular methods for making an algorithm $(\epsilon,0)$-differentially private is the Laplace mechanism \citep{dwork2006calibrating}. %, in which we add Laplace random noise to the output of the algorithm. 
%and (b) the exponential mechanism \citep{mcsherry2007mechanism}, which draws a random output $\tilde{\lambda}$ such that $\tilde{\lambda} \approx \lambda$. 
For this mechanism we must define an intermediate quantity called the \textbf{global sensitivity}, $\Delta_{\Ac}$ describing how much $\Ac$ changes when $\Dc$ changes,
\begin{align}
\Delta_{\Ac} := \max_{\Dc,\tilde{\Dc} \subseteq \Xc \textrm{ s.t. } d_H(\Dc,\tilde{\Dc})\leq 1} | \Ac(\Dc) - \Ac(\tilde{\Dc}) |.\nonumber
\end{align}
%%where $\Xc$ is the set 
%
% \begin{define}
% The \textbf{global sensitivity} of an algorithm $\Ac$ over all neighboring datasets $\Dc,\tilde{\Dc}$ (i.e., $\Dc,\tilde{\Dc}$ differ by the value of one record) is
% \begin{align}
% \Delta_{\Ac} \triangleq \max_{\Dc,\tilde{\Dc} \subseteq \Xc } | \Ac(\Dc) - \Ac(\tilde{\Dc}) |. \nonumber
% \end{align}
% \end{define}
% given an algorithm $\Ac(\Vc) = \argmax_{\lambda \in \Lambda} q(\Vc,\lambda)$
%. Before defining each mechanism, let $\Ac$ be an algorithm we wish to make private. We introduce the global sensitivity of $\Ac$, which describes how much $\Ac$ can change when $\Vc$ changes.
%
%\begin{define}
%The \textbf{global sensitivity} of $\Ac$ over all possible neighboring datasets $\Vc,\Vc'$ is:
%\begin{align}
%\Delta_{\Ac} \triangleq \max_{\Vc,\Vc' } | \Ac(\Vc) - \Ac(\Vc') | \nonumber
%\end{align}
%\end{define}
The Laplace mechanism hides the output of $\Ac$ with a small amount of additive random noise, large enough to hide the impact of any \emph{single} datum $(x_i,y_i)$. 
%The Laplace mechanism is useful for when we would like to release the output of $\Ac$ on dataset $\Vc$ privately.
%
\begin{define}
Given a dataset $\Dc$ and an algorithm $\Ac$, the \textbf{Laplace mechanism} returns $\Ac(\Dc) + \omega$, where $\omega$ is a noise variable drawn from $\Lap(0,\Delta_{\Ac} / \epsilon)$, the Laplace distribution with scale parameter $\Delta_{\Ac} / \epsilon$.
\end{define}
%geq 0$ (as we could swap $\Vc$ and $\Vc'$ in the definition).
%
It may be that the global sensitivity of an algorithm $\Ac$ is unbounded in general, but can be bounded in the context of a specific data set $\Dc$ over all neighbors $\tilde{\Dc}$. %However, there may exist a dataset $\dot{\Dc}$ in which the algorithm has bounded sensitivity over all neighbors of $\dot{\Dc}$. 
For such datasets we can bound the \textbf{local sensitivity} 
%$\Delta({\Dc})_{\Ac}$,
%
\begin{align}
\Delta({\Dc})_{\Ac} := \max_{\tilde{\Dc} \subseteq \Xc \textrm{ s.t. } d_H(\Dc,\tilde{\Dc})\leq 1} | \Ac(\Dc) - \Ac(\tilde{\Dc}) |. \nonumber
\end{align}
If an algorithm has bounded global sensitivity it certainly has bounded local sensitivity. \citet{nissim2007smooth,dwork2009differential,jain2013differentially} show how to use the local sensitivity to cleverly produce private quantities for datasets with bounded local sensitivity.

\section{Test Set Privacy}
\label{sec:testprivate}
%!TEX root=pci.tex
%We present methods for ensuring the privacy of the test and training sets used in section~\ref{alg:anm_ci}. Furthermore, we demonstrate that the private quantities released very likely do not change the direction of causal inference.

The data is partitioned into training and test set, which are used in different ways. We therefore introduce mechanisms to preserve training and test set privacy respectively, which can be used jointly. Specifically, we show how to privatize the dependence scores $s_{X \rightarrow Y},s_{Y \rightarrow X}$. The reason for this is four-fold: 1. Privatizing the dependence score immediately privatizes the causal direction $D$, because operations on differentially private outputs preserve privacy (so long as they do not touch the data). 2. Releasing the scores indicates how confident the ANM method is about the causal direction, which is absent from the binary output $D$. 3. It is unclear which dependence score is best for a particular dataset, so we privatize multiple scores and leave this choice to the practitioner. In this section we begin with test set privacy and describe training set privacy in Section~\ref{sec:trainprivate}. Table~\ref{table.scores} gives an overview of test and training set privacy results for the dependence scores that we consider. 

Let $(\x',\y')$ be the initial test data and $(\xbt',\ybt')$ be the test data after a single change in the dataset. Let $\xbt' = [x_1', \ldots, x_{k-1}', \tilde{x}_k', x_{k+1}', \ldots, x_m']^\top$ and similarly for $\ybt$ so that this single change occurs at some index $k$. The key to preserving privacy is to show that the selected dependence score $s(\cdot,\cdot)$ can be privatized. We show that if our dependence score is a rank correlation coefficient (Spearman's $\rho$, Kendall's $\tau$) or the HSIC score \cite{gretton2005measuring}, we can readily bound its test set global sensitivity when applied to $(\x',\y')$ versus $(\xbt',\ybt')$. 
 %Similarly, because of the nice properties of the HSIC score \cite{gretton2005measuring}, we can bound its test set global sensitivity as well.
 %We then show the same condition holds for the HSIC score based on the fact that it first transforms the data by a bounded kernel function and then applies a Lipschitz function to this transformation. 
As the IQR score has bounded test set local sensitivity we can apply the algorithm of \citet{dwork2009differential} for privacy. 
%We then show that some of the most common estimators of the variance and entropy have unbounded local and global sensitivities, implying that scores based on these quantities \cite{buhlmann2014cam,kpotufe2014consistency} can not be easily privatized. 
 % MML \cite{stegle2010probabilistic}, 
%Variance \cite{buhlmann2014cam} and Entropy score \cite{kpotufe2014consistency} have common estimators with unbounded local and global sensitivities, and so therefore cannot be readily privatized.

%We first demonstrate that if a dependence score is Lipschitz w.r.t. to a bounded function of the data then we can immediately preserve the privacy of the underlying data. Additionally, we can show that with high-probability the inferred causal direction will not change as a result. Even if this is not the case, we can obtain similar guarantees if the score is instead based on a robust statistic, using the framework of \citet{dwork2009differential}.

\begin{table}
\vspace{-2.5ex}
\caption{Dependence scores and their privacy. A checkmark indicates that there exist meaningful bounds on either the global or local sensitivity.}
\vspace{-1ex}
\label{table.scores}
\begin{center}
%\resizebox{0.5\textwidth}{!}
\small
{
\begin{sc}
\begin{tabular}{c|cc|cc}
%\hline
%\multicolumn{5}{c}{\textbf{Dataset Statistics}}\\
%& \multicolumn{2}{c}{\textbf{Private}} & \textbf{Unknown} \\
\hline
 & \multicolumn{2}{c|}{\textbf{Test}} & \multicolumn{2}{c}{\textbf{Training}} \\
 & {Global} & {Local} & {Global} & {Local} \\ %&  \\
\textbf{Score} & {Sense.} & {Sense.} & {Sense.} & {Sense.} \\ \hline%& \textbf{Neither}\footnote{For the estimators considered.} \\ \hline
\hline
Spearman's $\rho$ &  $\checkmark$ & $\checkmark$ & - & $\checkmark$  \\ \hline %& - \\ \hline
Kendall's $\tau$ & $\checkmark$ & $\checkmark$ & - & $\checkmark$ \\ \hline %& - \\ \hline
HSIC & $\checkmark$ & $\checkmark$ & $\checkmark$ & $\checkmark$  \\ \hline%& - \\ \hline
IQR & - & $\checkmark$ & - & $\checkmark$ \\ \hline %& - \\ \hline
%MML & - & $\times$ & - \\ \hline
%Variance & - & - & $\times$ \\ \hline
%Entropy & - & - & $\times$ \\ \hline
\end{tabular}
\end{sc}}
\end{center}
\vspace{-4ex}
\end{table}

% Note that $\x'$ and $\xbt'$ differ in at most one element (similarly for $\y'$ and $\ybt'$). The length of both test datasets is $m$. 

\subsection{Rank Correlation Coefficients}
We first demonstrate global sensitivity for the two rank correlation scores in Section~\ref{sec:background}.

\begin{thm}
The rank correlation coefficients have the following global sensitivities,
\begin{enumerate}
\item Let $\rho(\cdot,\cdot)$ be Spearman's $\rho$ score, then
\begin{align}
| \rho(\x',\rr_Y') - \rho(\xbt',\rbt_Y')| \leq \frac{30}{m} \nonumber
\end{align}
\item Let $\tau(\cdot,\cdot)$ be Kendall's $\tau$ score, then
\begin{align}
| \tau(\x',\rr_Y') - \tau(\xbt',\rbt_Y')| \leq \frac{4}{m} \nonumber
\end{align}
\end{enumerate}
\end{thm}

\begin{proof}
Our goal is to bound the following global sensitivity in both scores: $| s(\x',\rr_Y') - s(\xbt',\rbt_Y')|$. For Spearman's $\rho$, suppose the change is on $a_k$ and $b_k$, it is easy to verify that 1) $d_i$ changes by at most $2$, for $i\neq k$; 2) $d_k$ changes by at most $m-1$; 3) $d_i \leq m-1$ for all $i$. Since $d_i^2 - (d_i-2)^2 = 4(d_i-1) \leq 4(m-2)$ for $i\neq k$, the maximum change inside the summation is upper bounded by $(m-1)(4m-8) + (m-1)^2$. Therefore, global sensitivity of $\rho$ is bounded by
$$ \frac{6(m-1)(5m-3)}{m(m^2-1)} \leq \frac{30}{m} $$. 
% For Spearman's $\rho$ the largest change occurs at when $|d_i|$ is shrunk from its maximum values (as the derivative of $d_i^2$ increases as $|d_i|$ increases). Specifically, $d_i$ is maximized if the elements of $\x'$ are monotonically increasing ($x_1 \leq x_2 \leq \ldots \leq x_m$) and those of $\rr_Y'$ are monotonically decreasing ($r_{1,Y} \geq r_{2,Y} \geq \cdots \geq r_{m,Y}$). Using the triangle inequality we can bound how Spearman's $\rho$ changes when only $\x'$ changes to $\xbt'$ (and similarly for $\rr_Y'$ to $\rbt'_Y$). Thus the largest change in a single element is when $x_1$ changes from the smallest element in $\x'$ to the largest element in $\xbt'$. This reduces $d_1^2$ from $(m-1)^2$ to $0$ (as now the largest element in $\xbt'$ and $\rr_Y'$ occur at index 1). As well it reduces $d_j$ for $j \!\neq\! 1$ by $1$, implying that $d_j^2$ changes by at most $2 d_j - 1$, where $|d_j| \leq m-1$. Thus, over all $j \!\neq\! 1$, which happens $(m-1)$ times the change is at most $(m-1)(2m-3)$. Therefore, the sensitivity of $\x$ is bounded as such,
% \begin{align}
% | s(\x',\rr_Y') - s(\xbt',\rr_Y')| \leq& \frac{6[(m-1)(2m-3) + (m-1)^2]}{m(m^2-1)}  \nonumber \\
% \leq& \frac{18}{m} \nonumber
% \end{align}
% and thus the global sensitivity is bounded by $36/m$.

For Kendall's $\tau$ we can affect at most $(m\!-\!1)$ pairs by moving a single element of $\x'$, as well as $(m\!-\!1)$ pairs for changing $\rr_Y'$ (either from concordant pairs to discordant pairs, or vice versa). Therefore, the global sensitivity of Kendall's $\tau$ is
\begin{align}
| s(\x',\rr_Y') - s(\xbt',\rbt_Y')| \leq \frac{2(m-1)}{\frac{1}{2}m(m-1)} \leq \frac{4}{m} \nonumber
\end{align}
\end{proof}

The bound on the global sensitivity $\Delta$ of our scores enables us to apply the Laplace mechanism \cite{dwork2006calibrating} to produce $(2\epsilon,0)$-differentially private scores: $p_{X \rightarrow Y},p_{Y \rightarrow X}$. Specifically, we add Laplace noise $\mbox{Lap}(0,\Delta/\epsilon)$ to our Spearman's $\rho$ and Kendall's $\tau$ scores to preserve privacy w.r.t. the test set. 
% to ensure $(2\epsilon,0)$-differential privacy with respect to the test set. 
%where $\Delta \!=\! 18/m$ for Spearman's $\rho$ score and $\Delta \!=\! 4/m$ for Kendall's $\tau$.
 %\frac{12m - 11}{\epsilon(m-1)^2})$ 
%to our empirical HSIC scores to ensure $(2\epsilon,0)$-differential privacy with respect to the test set. 
Moreover, as a general property of differential privacy we can compute any functions on these private scores and, so long as they do not touch the data, the outputs of these functions are also private. This means that we can compute the inequality $p_{X \rightarrow Y}\!<\! p_{Y \rightarrow X}$ to decide if $X$ causes $Y$ or vice-versa privately.

An important consideration is to what degree the addition of noise affects the true decision: $s_{X \rightarrow Y}\! <\! s_{Y \rightarrow X}$. 
%We demonstrate that the exponentially decreasing tails of the Laplace distribution ensures that, so long as the difference between the non-noisy scores is large enough, and $m$ is large enough, the causal inference direction will not change with high probability.
Importantly, we can prove that, in certain cases, the addition of Laplace noise required by the mechanism is small enough to not change the direction of causal inference. These are cases in which there is a large `margin' between the scores $s_{X \rightarrow Y}$ and $s_{Y \rightarrow X}$. So long as this margin is large enough and in the correct order the addition of Laplace noise has no effect on the inference with high probability.

\begin{thm}
\label{thm:hsic_util}
Given two random variables $X,Y$ who have w.l.o.g. the causal relationship $X \rightarrow Y$, assume that they produce correctly-ordered scores: $s_{X \rightarrow Y} < s_{Y \rightarrow X}$, with margin $\gamma = s_{Y \rightarrow X} - s_{X \rightarrow Y}$. Let $p_{X \rightarrow Y}, p_{Y \rightarrow X}$ be these scores after applying the Laplace mechanism \cite{dwork2006calibrating} with scale $\sigma = \Delta/\epsilon$ %$\sigma = \frac{12m - 11}{\epsilon(m-1)^2}$, 
then the probability of correct inference with these private scores is,
\begin{align}
\mathbb{P}(p_{X \rightarrow Y} < p_{Y \rightarrow X}) = 1 - \frac{\gamma + 2\sigma}{4 \sigma} e^{-\frac{\gamma}{\sigma}} \nonumber.
\end{align}
%where $\sigma = \frac{12m - 11}{\epsilon(m-1)^2}$. 
%is the scale of Laplace noise added to form the private scores.
\end{thm}

We leave the proof to the appendix. %Recall that $\Delta$ is the global sensitivity of $s(\cdot,\cdot)$ and $\epsilon$ is the required level of privacy. 
Note that the probability of incorrect inference decreases nearly exponentially as the margin $\gamma$ increases. This is a particularly nice property as the margin essentially describes the confidence of the (non-private) causal inference prediction: large $\gamma$ corresponds to high confidence in the inference. Additionally, there is an exponential decrease as $m$ and $\epsilon$ grow. In section~\ref{sec:results}, we show on real-world causal inference data that we can accurately recover the true causal direction for a variety $\epsilon$ settings.

\subsection{HSIC Score}
We begin by defining the empirical estimate of the HSIC score given kernels $k,l$:
%In this subsection we demonstrate that the $\widehat{\mbox{HSIC}}_{k,l}$ (with kernels $k,l$) score: 
\begin{align}
\widehat{\mbox{HSIC}}_{k,l}(\x',\rr_Y') := \frac{1}{(m-1)^2} \mbox{trace}(KHLH) \label{eq:HSIC}
\end{align}
where $K_{ij} = k(x_i',x_j')$, $L_{ij} = l(r_{Y,i}',r_{Y,j})$ and $H_{ij} = \delta_{\{i=j\}} - 1/m$. %Let $k,l$ be upper bounded by $1$. As defined this score satisfies Condition~\ref{cond:lipschitz}. 
We assume $k,l$ are bounded above by $1$ (e.g., the squared exponential kernel, the Matern kernel \cite{rasmussen2006gaussian}). Our goal is to show that when we replace $(\x',\y')$ with $(\xbt',\ybt')$ the global sensitivity is small. Specifically we prove the following theorem.

\begin{thm}
\label{thm:hsic_sense}
The score in eq.~(\ref{eq:HSIC}) has a global sensitivity of at most $\frac{16m - 8}{(m-1)^2}$. Specifically,
\begin{align}
|\widehat{\mbox{HSIC}}_{k,l}(\x',\rr_Y') - \widehat{\mbox{HSIC}}_{k,l}(\xbt',\rbt_Y')| \leq \frac{16m - 8}{(m-1)^2} \nonumber
\end{align}
\end{thm}

\begin{proof}
For simplicity define $\mathcal{H}(\cdot,\cdot) := \widehat{\mbox{HSIC}}_{k,l}(\cdot,\cdot)$. 
Note that, as the trace is cyclic: $\mbox{trace}(KHLH) = \mbox{trace}(HKHL)$. Further, let $\tilde{K},\tilde{L}$ be the kernels defined on the modified data $(\xbt',\ybt')$. Then as the data is represented purely through the kernel matrices and the trace is Lipschitz w.r.t. these matrices, we can apply the triangle inequality to yield,
\begin{align}
| \mathcal{H}(\x',\rr_Y') - \mathcal{H}(\xbt',\rbt_Y')| \leq& \nonumber \\
\frac{\|HLH\|_{\infty} \| K - \tilde{K} \|_1}{(m-1)^2} +& \frac{\|HKH\|_{\infty} \| L - \tilde{L} \|_1 }{(m-1)^2} \nonumber
% &| \mathcal{H}(\x',\rr_Y') - \mathcal{H}(\xbt',\rr_Y') | \nonumber \\
%+ &| \mathcal{H}(\xbt',\rr_Y') - \mathcal{H}(\xbt',\rbt_Y')| \nonumber
\end{align}
To bound the infinity norms, let $\overline{L} = HLH$, then
\begin{align}
&| \overline{L}_{ij} | = \Bigg| L_{ij} - \frac{ \sum_{a=1}^m L_{aj} }{m} - \frac{ \sum_{b=1}^m L_{ib} }{m} + \frac{ \sum_{a,b = 1}^m L_{ab} }{m^2} \Bigg| \nonumber \\
&\leq 4 \nonumber
\end{align}
as $L_{ij} \leq 1$ (this inequality also holds for $HKH$). Finally, note that as there is only a single-element difference between $(\x',\rr_Y')$ and $(\xbt',\rbt_Y')$, we have that $\| K - \tilde{K} \|_1 \leq 2m - 1$ (and also for $L,\tilde{L}$).%, finishing the proof.
\end{proof}

% \begin{proof}
% Before bounding these terms note that $\mbox{trace}(KHLH) \!=\! \mbox{trace}(HKHL)$. For the first term we have,
% \begin{align}
% | \mathcal{H}(\x',\rr_Y') - \mathcal{H}(\xbt',\rr_Y') | = \frac{ \Big| \sum_{i,j=1}^m (K_{ij} - \tilde{K}_{ij}) \overline{L}_{ij} \Big|}{(m-1)^2} \nonumber
% \end{align}
% where $\overline{L} = HLH$. We can bound the numerator by observing that,
% \begin{align}
% &| \overline{L}_{ij} | = \Bigg| L_{ij} - \frac{ \sum_{a=1}^m L_{aj} }{m} - \frac{ \sum_{b=1}^m L_{ib} }{m} + \frac{ \sum_{a,b = 1}^m L_{ab} }{m^2} \Bigg| \nonumber \\
% &\leq \Big| L_{ij}  \Big| + \Bigg| \frac{ \sum_{a=1}^m L_{aj} }{m} \Bigg| +  \Bigg| \frac{ \sum_{b=1}^m L_{ib} }{m} \Bigg| + \Bigg| \frac{ \sum_{a,b = 1}^m L_{ab} }{m^2} \Bigg| \nonumber \\
% &\leq 4 \nonumber
% \end{align}
% as $l$ is bounded above by $1$. Further, the sum $\Big| \sum_{i,j=1}^m (K_{ij} - \tilde{K}_{ij}) \Big| \leq 2m-1$, as changing one element affects at most $2m-1$ entries of $K$ and $k$ is bounded. Therefore we have that,
% \begin{align}
% |\mathcal{H}(\x',\rr_Y') - \mathcal{H}(\xbt',\rr_Y') | \leq \frac{8m - 4}{(m-1)^2} \nonumber
% \end{align}
% Because we didn't use any properties of $\x'$ or $\rr_Y'$, just the properties of the kernels, we have the exact same bound for the second term $| \mathcal{H}(\xbt',\rr_Y') - \mathcal{H}(\xbt',\rbt_Y')|$, finishing the proof of Theorem~\ref{thm:hsic_sense}.
% \end{proof}

In fact, we can improve this bound to $\frac{12m - 11}{(m-1)^2}$ using trace identities. We leave the proof of this to the appendix. Given this global sensitivity bound we can use Theorem~\ref{thm:hsic_util} to guarantee that under certain conditions the Laplace mechanism w.h.p. does not change the direction of causal influence.

\subsection{IQR Score}
Unfortunately the IQR does not have a bounded global sensitivity, as there exist datasets for which the IQR can change by an unbounded amount. %we are not able to apply Theorem~\ref{thm:lipschitz} to release the IQR score described in Section~\ref{sec:background}, as it is not Lipschitz with-respect to a bounded quantity, such as a kernel. 
Instead, \citet{dwork2009differential} offer an efficient technique to privately release the IQR. We give a slightly modified version of their Algorithm in the appendix. % Algorithm~\ref{alg:iqr_ptr}.
%%%%%%%%%%%%%%%We will use the propose-test-release framework of \citet{dwork2009differential} to preserve the privacy of the IQR computed on a finite sample. We give a brief description of how this is done leaving additional details to the original paper \cite{dwork2009differential}. %Algorithm~\ref{alg:iqr_ptr} gives the propose-test-release algorithm for the IQR.

First the algorithm defines two intervals $B_1$ and $B_2$ which both contain IQR$(\Xb)$. If the IQR were to be pushed out of both of these intervals it would imply that the IQR changed by a factor of $e$. Therefore we loop over both intervals and calculate the number of points $A_j$ that an adversary would need to change to push the IQR out of $B_1$ or $B_2$. Note that $A_j$ is itself a data-sensitive query and so, to preserve privacy of this query, we can add Laplace noise to it. Then, if one of these noisy estimates $R_j = A_j + z$, where $z \sim \mbox{Lap}(0,1/\epsilon)$ is larger than some threshold, it implies that with high probability (exactly $1-\delta$), that the IQR$(\Xb)$ has multiplicative sensitivity of at most $e$, \emph{for the specific dataset $\Xb$}. Note that this is precisely the local sensitivity as defined in Section~\ref{sec:background}, as it is specific to $\Xb$. This means that we can add Laplace noise $z$ to $\log \mbox{IQR}(\Xb)$. 
%, or equivalently multiply IQR$(\Xb)$ by $2^z$, where $z \sim \mbox{Lap}(0,1/\epsilon)$. 
If neither of the $R_j$ are above the threshold then the algorithm returns null: $\perp$. This algorithm was shown to be $(3\epsilon,\delta)$-differentially private.

%\citet{dwork2009differential} note that if we replace $\log_{2(\cdot)}$ and $2^{(\cdot)}$ with $\log_{(1 + 1 / \log n)}$ and $(1 + 1 / \log n)^{(\cdot)}$ we obtain much better utility results. Therefore in the analysis that follows we assume Algorithm~\ref{alg:iqr_ptr} is run with these replacements.

%First the algorithm computes how many datapoints an adversary would need to alter to significantly affect the IQR (i.e., make the IQR smaller or larger by a factor of at least $(1 + \frac{1}{\log m})$). Let this answer to this query be $A$ (this can be easily computed by ordering the dataset). This itself is a data-sensitive query and so Laplacian noise is added to this estimate. Then, so long as this noisy estimate is larger than some threshold we can release a private IQR by multiplying it by $(1 + \frac{1}{\log m})^z$, where $z \sim \mbox{Lap}(0,1/\epsilon)$, as the IQR is likely to stay within the aforementioned range. If not, then the algorithm returns null: $\perp$. This algorithm was shown to be $(4\epsilon,\delta)$-differentially private. 

In our case we would like to release four private IQR scores. Note that we must look at $\x'$ three separate times: for $\mbox{IQR}(\x'),\mbox{IQR}(\rr_Y'),$ and $\mbox{IQR}(\rr_X')$ (and three times as well for $\y'$). Therefore for both $\x'$ and $\y'$ we are composing three differentially private outputs. Under simple composition this would lead to $(9\epsilon,3\delta)$ differential privacy for both $\x'$ and $\y'$. However, we can make use of Corollary 3.21 in \citet{dwork2013algorithmic} to give $(\epsilon', 3\delta + \delta')$-differential privacy, for $0 < \epsilon' < 1$ and $\delta' > 0$, over three repeated mechanisms by ensuring each private mechanism is $(3\epsilon,\delta)$-private, where $3\epsilon = \epsilon' / (2 \sqrt{6 \log(1/\delta')})$.

The remaining question is whether this noise addition causes one to infer the incorrect causal direction. Again, as long as there is a significant margin between the scores, we can preserve the correct causal inference with high probability as follows.
\begin{thm}
\label{thm:iqr}
Let $Q_{\x'} = \log IQR(\x')$, %_{(1 + \frac{1}{\log m})} 
and similarly for $Q_{\y'},Q_{\rr_X'},Q_{\rr_Y'}$, be the true log-IQR scores. As well let $P_{\x'},P_{\y'},P_{\rr_X'},P_{\rr_Y'}$ be the private versions, multiplied by $e^z$ noise where $z \sim \mbox{Lap}(0,1/\epsilon)$. The the following results hold:
\begin{enumerate}
\item \cite{dwork2009differential} If the number of data-points needed to significantly change the IQR, $A_j$, is less than $e$ 
%$(1 + \frac{1}{\log m})$ 
then, the probability that any one of the private IQR $P_{*}$ is released is small:
\begin{align}
\mathbb{P}\Bigg[ P_{*} \neq \perp | A_1 \mbox{ or } A_2 \leq e \Bigg] \leq \frac{3\delta}{2}. \nonumber
\end{align} 
\item If all private log-IQR scores are released, and the relationship between the true scores holds $Q_{\x'} + Q_{\rr_Y'} < Q_{\y'} + Q_{\rr_X'}$ (which implies $X \rightarrow Y$), then the probability that we make the correct causal inference from the private scores is large,
\begin{align}
&\mathbb{P}[P_{\x'} + P_{\rr_Y'} < P_{\y'} + P_{\rr_X'}] = \nonumber \\
& 1 - \frac{e^{\frac{-\gamma}{\sigma}}}{96 \sigma^3} \Big( 48 \sigma^3 + 33 \sigma^2 \gamma + 9 \sigma \gamma^2 + \gamma^3  \Big)\nonumber
\end{align}
where $\gamma = Q_{\y'} + Q_{\rr_X'} -  Q_{\x'} + Q_{\rr_Y'}$, and $\sigma = 1/\epsilon$.%\log(1 + \frac{1}{n})/\epsilon$.
\end{enumerate}
\end{thm}

The proof of these results is in the appendix. The first result says that the probability that we release an IQR score just because too much noise was added to $A_j$ is small. The second result says that with high probability we recover the true causal direction, depending on the size of the dataset.
%We can immediately preserve the privacy of the IQR score using the propose-test-release framework of \citet{dwork2009differential}. They show how to release a private IQR score by first (privately) testing the sensitivity of the IQR. If the senstivity is `sufficiently small' a slightly noisy version of the IQR is released. Otherwise, a failure $\perp$ is returned. Furthermore, as any differentially private quantity retains privacy after applying a data-ignorant function (i.e., $\log$), the score in eq.~(\ref{eq:iqr_score}) is private. For more details on propose-test-release see \citet{dwork2013algorithmic}.

\section{Training Set Privacy}
\label{sec:trainprivate}
%!TEX root=pci.tex

Let $(\x,\y)$ be the initial training data and $(\xbt,\ybt)$ be the training data after a change in the dataset. Note that $\x$ and $\xbt$ differ in at most one element (similarly for $\y$ and $\ybt$). The length of both training datasets is $n$. From Algorithm \ref{alg:anm_ci}, the only way the training set can affect the dependency scores $s_{X \rightarrow Y}, s_{Y \rightarrow X}$ is through the regression functions $\hat{f},\hat{g}$, used to compute test set residuals $\mathbf{r}_Y',\mathbf{r}_X'$. We use the kernel ridge regression method and so the functions $\hat{f}$ (and $\hat{g}$) can be written in the form: $\hat{f}(\w,x) = \w^\top \phi(x)$, where $\phi(x)$ is a (possibly infinite) feature space mapping to the Hilbert space corresponding to the kernel function used. Similar to other work on private regression \cite{talwar2014private} we assume that $|x|, |y| \leq 1$. The ridge regression algorithm can now be written as:
\begin{align}
\w = \argmin_{\w \in \mathcal{H}} \frac{\lambda}{2} \| \w \|_{\mathcal{H}}^2  + \frac{1}{n} \sum_{i=1}^n (\w^\top \phi(x_i) - y_i)^2, \label{eq:kernel_ERM}
\end{align}
where $\mathcal{H}$ is the corresponding Hilbert space. 
%We assume the loss is convex and $L_\ell$-Lipschitz. Examples of such losses are the absolute loss $\ell(\w,\phi(x),y) = |\w^\top \phi(x) - y|$ and the Huber loss:
% \begin{align}
% \ell(\w,\phi(x),y) = 
% \begin{cases}
% \frac{1}{2}(\w^\top \phi(x)  - y)^2 & \mbox{ for } |\w^\top \phi(x) - y| \leq 1 \\
% |\w^\top \phi(x) - y| - \frac{1}{2} & \mbox{ otherwise }
% \end{cases}
% \end{align}
%for both of which we have $L_\ell = 1$.
Practically speaking, even though $\w$ may be infinite-dimensional, because it always appears in an inner product with the feature mapping $\phi(x)$ we can utilize the `kernel trick': $k(x_i,x_j) = \phi(x_i)^\top \phi(x_j)$ to avoid having to represent $\w$ explicitly. 

Let $\hat{f}(\w^*, \cdot)$ and $\hat{f}(\tilde{\w}^*,\cdot)$ be the classifiers resulting from the optimization problem in eq.~(\ref{eq:kernel_ERM}) when trained on $(\x,\y)$ and $(\xbt,\ybt)$, respectively (and similarly for $\hat{g}$). We show that the residuals in Algorithm~\ref{alg:anm_ci} are bounded.

\begin{thm}
\label{thm:residual_bound}
Say $\lambda \le 1$.  Given that the classifiers $\hat{f}(\w^*,\cdot),\hat{f}(\tilde{\w}^*,\cdot)$ are the result of the optimization problem in eq.~(\ref{eq:kernel_ERM}),  the residuals of these functions $\rr_Y',\rbt_Y'$ are bounded as,
\begin{align}
%\| r_{i,Y}' - \tilde{r}_{i,Y}' \|_2 \leq (1 + \frac{1}{\sqrt{\lambda}})\frac{2 \sqrt{m}}{\lambda n} \label{eq:residual_bound}
| r_{i,Y}' - \tilde{r}_{i,Y}' | \leq \frac{8}{n \lambda^{3/2}}   %\max\{\frac{1}{x},\frac{1}{\sqrt{x}}\}
%\Big(1 + \frac{1}{\sqrt{\lambda}}\Big)\frac{2}{\lambda n} 
\label{eq:residual_bound}
\end{align}
for all $i$, where $r_{i,Y}',\tilde{r}_{i,Y}'$ are the $i$th elements of $\rr_Y',\rbt_Y'$ and $m$ is the size of the test set.
\end{thm}

This bound holds equally for $\rr_X',\rbt_X'$. The proof of the above is inspired by the work of \citet{shalev2009stochastic} and \citet{jain2013differentially}. We place the proof in the appendix for the interested reader. As far as we are aware this is the tightest bound for the optimization problem in eq.~(\ref{eq:kernel_ERM}), with a non-Lipschitz loss. In the following, we use this bound to preserve training set privacy for the dependence scores considered in the previous section.

% for any $x$, where $C_k = \max_{x} k(x,x)$. From Algorithm \ref{alg:anm_ci}, this bound implies the same sensitivity for each individual test set residual, defined as $r_{i,Y}' = y_i' - \hat{f}(\w^*, x_i')$ (similarly for $r_{i,X}'$. Further, we have,
% \begin{align}
% \| \rr_Y' - \rbt_Y' \|_2 \leq (1 + \frac{1}{\sqrt{\lambda}})\frac{2 \sqrt{m}}{\lambda n} \label{eq:residual_bound}
% \end{align}
% where $m$ is the size of the test set. Using these bounds we can bound the change in the above dependence scores caused by the change in the training set: $(\x,\y)$ to $(\xbt,\ybt)$.

\begin{figure*}[t!]
\begin{center}
%\vspace{-2ex}
\centerline{\includegraphics[width=\textwidth]{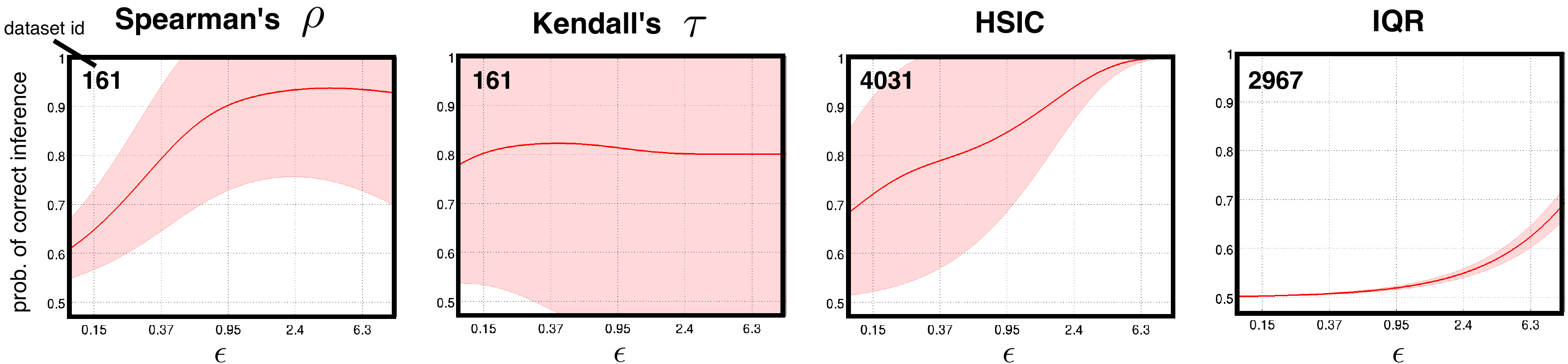}}
\vspace{-2ex}
\caption{Probability of correctly identifying the causal direction on datasets selected from the Cause-Effect Pairs Challenge \cite{guyonchallenge2013}. Datasets for which the scores perform well were selected in order to isolate the effect of privatization on the scores.}
%Privacy and utility for both Spearman's $\rho$ and Kendall's $\tau$.}
%\vspace{-10ex}
\label{figure.utility}
\end{center}
\vspace{-4ex}
\end{figure*}

\begin{figure*}[th!]
\begin{center}
\vspace{-2ex}
\centerline{\includegraphics[width=\textwidth]{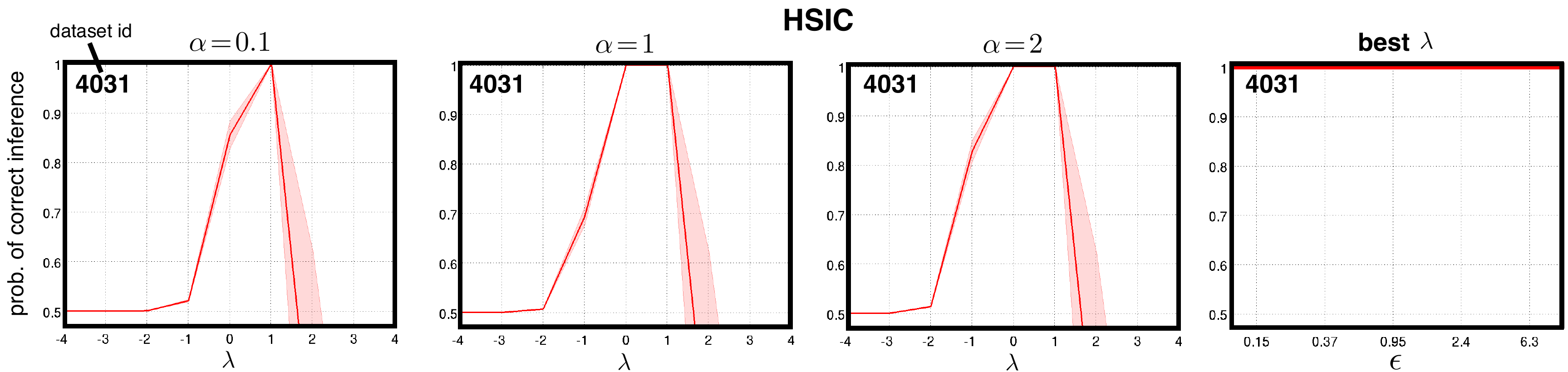}}
\vspace{-2ex}
\caption{Training set privacy for the HSIC score. The three left-most plots show how $\lambda$ affects the probability of correctly inferring the causal direction, while the right-most plot depicts this probability when the best $\lambda$ is selected over a $\epsilon \in [0.1,10]$. See text for more details. }
%Privacy and utility for both Spearman's $\rho$ and Kendall's $\tau$.}
%\vspace{-10ex}
\vspace{-5ex}
\label{figure.utility_tr}
\end{center}
\end{figure*}

%!TEX root=pci.tex
\begin{table*}[t]
\vspace{-4ex}
\caption{The non-private accuracies of the ANM model on a subset of the Cause-Effect Pairs Challenge \cite{guyonchallenge2013}, as well as the probability of correct causal inference after privatization.}%accuracy of all scores with/without privacy on a subset of the Cause-Effect Pairs Challenge \cite{guyonchallenge2013}.}
\vspace{-3ex}
\label{table.error}
\begin{center}
\resizebox{\textwidth}{!}
{ 
\begin{sc}
\footnotesize
\begin{tabular}{c||c|c|c|c|c|c|c|c|c|c}%c|c|c|c|c|c|c|c|c|c}
\hline
%\multicolumn{5}{c}{\textbf{Lower Bounds}}\\
\hline
dataset ids  & 4031    &     597    &    2209     &   2967      &   161      &  2132    &    1656    &     901     &   3484     &   1627 \\ \hline
size & 7713    &    7748     &   7766     &   7771    &    7782    &    7784     &   7803     &   7820     &   7853    &    7862  \\ \hline %&      7875     &   7889    &    7926     &   7931     &   7932     &   7938    &    7955      &  7956    &    7958     &   7958 \\ \hline
\hline
& \multicolumn{10}{c}{$\epsilon = \infty$ (non-private accuracies)} \\
\hline
% full
% spear
\textbf{Spearman's} $\rho$                & $0.50 \pm 0.53$ & $0.00 \pm 0.00$ & $0.00 \pm 0.00$ & $0.70 \pm 0.48$ & $0.90 \pm 0.32$ & $1.00 \pm 0.00$ & $0.00 \pm 0.00$ & $0.30 \pm 0.48$ & $0.00 \pm 0.00$ & $1.00 \pm 0.00$ \\ \hline %& $1.00 \pm 0.00$ & $0.30 \pm 0.48$ & $0.50 \pm 0.53$ & $0.00 \pm 0.00$ & $1.00 \pm 0.00$ & $0.10 \pm 0.32$ & $0.70 \pm 0.48$ & $0.60 \pm 0.52$ & $0.20 \pm 0.42$ & $0.00 \pm 0.00$ \\ \hline
% kend  % kend\\ \hline %               
\textbf{Kendall's} $\tau$                 & $0.50 \pm 0.53$ & $0.00 \pm 0.00$ & $0.00 \pm 0.00$ & $0.70 \pm 0.48$ & $0.80 \pm 0.42$ & $1.00 \pm 0.00$ & $0.00 \pm 0.00$ & $0.80 \pm 0.42$ & $0.00 \pm 0.00$ & $1.00 \pm 0.00$ \\ \hline %& $1.00 \pm 0.00$ & $0.20 \pm 0.42$ & $0.50 \pm 0.53$ & $0.00 \pm 0.00$ & $1.00 \pm 0.00$ & $0.10 \pm 0.32$ & $0.70 \pm 0.48$ & $0.60 \pm 0.52$ & $0.20 \pm 0.42$ & $0.00 \pm 0.00$ \\ \hline
% hsic  % hsic\\ \hline %
\textbf{HSIC} \cite{gretton2005measuring} & $1.00 \pm 0.00$ & $0.00 \pm 0.00$ & $1.00 \pm 0.00$ & $1.00 \pm 0.00$ & $0.70 \pm 0.48$ & $0.60 \pm 0.52$ & $1.00 \pm 0.00$ & $0.40 \pm 0.52$ & $1.00 \pm 0.00$ & $0.10 \pm 0.32$ \\ \hline %& $0.90 \pm 0.32$ & $0.50 \pm 0.53$ & $1.00 \pm 0.00$ & $1.00 \pm 0.00$ & $0.00 \pm 0.00$ & $0.90 \pm 0.32$ & $1.00 \pm 0.00$ & $1.00 \pm 0.00$ & $1.00 \pm 0.00$ & $1.00 \pm 0.00$ \\ \hline 
% iqr  % iqr\\ \hline %
\textbf{IQR} \cite{buhlmann2014cam}       & $0.50 \pm 0.53$ & $0.00 \pm 0.00$ & $0.10 \pm 0.32$ & $1.00 \pm 0.00$ & $1.00 \pm 0.00$ & $1.00 \pm 0.00$ & $0.00 \pm 0.00$ & $0.90 \pm 0.32$ & $0.00 \pm 0.00$ & $1.00 \pm 0.00$ \\ \hline %& $0.10 \pm 0.32$ & $0.40 \pm 0.52$ & $1.00 \pm 0.00$ & $0.00 \pm 0.00$ & $1.00 \pm 0.00$ & $0.00 \pm 0.00$ & $0.00 \pm 0.00$ & $0.00 \pm 0.00$ & $1.00 \pm 0.00$ & $1.00 \pm 0.00$ \\ \hline

\hline
& \multicolumn{10}{c}{$\epsilon = 0.1$} \\
\hline

% eps=0.1
% spear
\textbf{Spearman's} $\rho$ & $0.56 \pm 0.45$ & $0.03 \pm 0.00$ & $0.20 \pm 0.02$ & $0.57 \pm 0.10$ & $0.61 \pm 0.06$ & $0.92 \pm 0.02$ & $0.40 \pm 0.06$ & $0.34 \pm 0.21$ & $0.01 \pm 0.00$ & $0.82 \pm 0.02$ \\ \hline %& $0.55 \pm 0.11$ & $0.47 \pm 0.15$ & $0.50 \pm 0.05$ & $0.04 \pm 0.02$ & $0.92 \pm 0.02$ & $0.45 \pm 0.15$ & $0.70 \pm 0.47$ & $0.52 \pm 0.06$ & $0.32 \pm 0.22$ & $0.22 \pm 0.08$ \\ \hline 
% kend\\ \hline %
\textbf{Kendall's} $\tau$  & $0.54 \pm 0.48$ & $0.00 \pm 0.00$ & $0.00 \pm 0.00$ & $0.69 \pm 0.38$ & $0.78 \pm 0.24$ & $1.00 \pm 0.00$ & $0.12 \pm 0.09$ & $0.76 \pm 0.41$ & $0.00 \pm 0.00$ & $1.00 \pm 0.00$ \\ \hline %& $0.69 \pm 0.11$ & $0.36 \pm 0.37$ & $0.48 \pm 0.24$ & $0.00 \pm 0.00$ & $1.00 \pm 0.00$ & $0.23 \pm 0.28$ & $0.70 \pm 0.48$ & $0.61 \pm 0.27$ & $0.17 \pm 0.37$ & $0.00 \pm 0.01$ \\ \hline 
% hsic\\ \hline %
\textbf{HSIC} \cite{gretton2005measuring}      & $0.68 \pm 0.17$ & $0.49 \pm 0.00$ & $0.60 \pm 0.01$ & $0.50 \pm 0.00$ & $0.50 \pm 0.01$ & $0.50 \pm 0.00$ & $0.52 \pm 0.00$ & $0.43 \pm 0.06$ & $0.66 \pm 0.03$ & $0.50 \pm 0.00$ \\ \hline %& $0.51 \pm 0.01$ & $0.50 \pm 0.00$ & $0.50 \pm 0.00$ & $0.68 \pm 0.01$ & $0.41 \pm 0.01$ & $0.52 \pm 0.01$ & $0.61 \pm 0.06$ & $0.51 \pm 0.00$ & $0.73 \pm 0.02$ & $0.66 \pm 0.01$ \\ \hline 
% iqr\\ \hline %
\textbf{IQR} \cite{buhlmann2014cam}       & $0.50 \pm 0.00$ & $0.50 \pm 0.00$ & $0.50 \pm 0.00$ & $0.50 \pm 0.00$ & $0.51 \pm 0.00$ & $0.50 \pm 0.00$ & $0.50 \pm 0.00$ & $0.50 \pm 0.00$ & $0.50 \pm 0.00$ & $0.50 \pm 0.00$ \\ \hline %& $0.50 \pm 0.00$ & $0.50 \pm 0.00$ & $0.50 \pm 0.00$ & $0.50 \pm 0.00$ & $0.50 \pm 0.00$ & $0.50 \pm 0.00$ & $0.46 \pm 0.01$ & $0.49 \pm 0.00$ & $0.51 \pm 0.00$ & $0.50 \pm 0.00$ \\ \hline

\hline
& \multicolumn{10}{c}{$\epsilon = 1$} \\
\hline

% eps=1
% spear
\textbf{Spearman's} $\rho$ & $0.50 \pm 0.53$ & $0.00 \pm 0.00$ & $0.00 \pm 0.00$ & $0.69 \pm 0.43$ & $0.91 \pm 0.17$ & $1.00 \pm 0.00$ & $0.06 \pm 0.07$ & $0.30 \pm 0.41$ & $0.00 \pm 0.00$ & $1.00 \pm 0.00$ \\ \hline %& $0.70 \pm 0.12$ & $0.32 \pm 0.39$ & $0.46 \pm 0.32$ & $0.00 \pm 0.00$ & $1.00 \pm 0.00$ & $0.14 \pm 0.31$ & $0.70 \pm 0.48$ & $0.62 \pm 0.35$ & $0.20 \pm 0.42$ & $0.00 \pm 0.00$ \\ \hline
% kend% kend\\ \hline %
\textbf{Kendall's} $\tau$  & $0.50 \pm 0.53$ & $0.00 \pm 0.00$ & $0.00 \pm 0.00$ & $0.70 \pm 0.48$ & $0.81 \pm 0.40$ & $1.00 \pm 0.00$ & $0.00 \pm 0.00$ & $0.80 \pm 0.42$ & $0.00 \pm 0.00$ & $1.00 \pm 0.00$ \\ \hline %& $1.00 \pm 0.01$ & $0.25 \pm 0.42$ & $0.46 \pm 0.47$ & $0.00 \pm 0.00$ & $1.00 \pm 0.00$ & $0.10 \pm 0.32$ & $0.70 \pm 0.48$ & $0.64 \pm 0.47$ & $0.20 \pm 0.42$ & $0.00 \pm 0.00$ \\ \hline
% hsic% hsic\\ \hline %
\textbf{HSIC} \cite{gretton2005measuring}      & $0.85 \pm 0.16$ & $0.39 \pm 0.03$ & $0.98 \pm 0.00$ & $0.52 \pm 0.01$ & $0.55 \pm 0.06$ & $0.50 \pm 0.01$ & $0.66 \pm 0.02$ & $0.21 \pm 0.25$ & $1.00 \pm 0.01$ & $0.49 \pm 0.01$ \\ \hline %& $0.55 \pm 0.06$ & $0.50 \pm 0.01$ & $0.54 \pm 0.01$ & $1.00 \pm 0.00$ & $0.03 \pm 0.01$ & $0.68 \pm 0.08$ & $0.95 \pm 0.06$ & $0.61 \pm 0.02$ & $1.00 \pm 0.00$ & $1.00 \pm 0.00$ \\ \hline
% iqr% iqr\\ \hline %
\textbf{IQR} \cite{buhlmann2014cam}       & $0.54 \pm 0.04$ & $0.48 \pm 0.00$ & $0.49 \pm 0.00$ & $0.52 \pm 0.00$ & $0.58 \pm 0.01$ & $0.51 \pm 0.01$ & $0.48 \pm 0.00$ & $0.50 \pm 0.00$ & $0.47 \pm 0.01$ & $0.51 \pm 0.00$ \\ \hline %& $0.50 \pm 0.00$ & $0.50 \pm 0.00$ & $0.51 \pm 0.00$ & $0.48 \pm 0.01$ & $0.51 \pm 0.00$ & $0.48 \pm 0.00$ & $0.17 \pm 0.03$ & $0.44 \pm 0.00$ & $0.59 \pm 0.02$ & $0.52 \pm 0.00$ \\ \hline

\hline
& \multicolumn{10}{c}{$\epsilon = 2$} \\
\hline

% eps=2
% spear
\textbf{Spearman's} $\rho$ & $0.50 \pm 0.53$ & $0.00 \pm 0.00$ & $0.00 \pm 0.00$ & $0.69 \pm 0.47$ & $0.93 \pm 0.17$ & $1.00 \pm 0.00$ & $0.01 \pm 0.02$ & $0.31 \pm 0.45$ & $0.00 \pm 0.00$ & $1.00 \pm 0.00$ \\ \hline %& $0.81 \pm 0.10$ & $0.27 \pm 0.41$ & $0.45 \pm 0.40$ & $0.00 \pm 0.00$ & $1.00 \pm 0.00$ & $0.10 \pm 0.31$ & $0.70 \pm 0.48$ & $0.64 \pm 0.43$ & $0.20 \pm 0.42$ & $0.00 \pm 0.00$ \\ \hline
% kend% kend\\ \hline %
\textbf{Kendall's} $\tau$  & $0.50 \pm 0.53$ & $0.00 \pm 0.00$ & $0.00 \pm 0.00$ & $0.70 \pm 0.48$ & $0.80 \pm 0.42$ & $1.00 \pm 0.00$ & $0.00 \pm 0.00$ & $0.80 \pm 0.42$ & $0.00 \pm 0.00$ & $1.00 \pm 0.00$ \\ \hline %& $1.00 \pm 0.00$ & $0.25 \pm 0.42$ & $0.47 \pm 0.50$ & $0.00 \pm 0.00$ & $1.00 \pm 0.00$ & $0.10 \pm 0.32$ & $0.70 \pm 0.48$ & $0.61 \pm 0.50$ & $0.20 \pm 0.42$ & $0.00 \pm 0.00$ \\ \hline
% hsic% hsic\\ \hline %
\textbf{HSIC} \cite{gretton2005measuring}      & $0.92 \pm 0.09$ & $0.29 \pm 0.04$ & $1.00 \pm 0.00$ & $0.55 \pm 0.01$ & $0.59 \pm 0.11$ & $0.51 \pm 0.01$ & $0.78 \pm 0.02$ & $0.20 \pm 0.26$ & $1.00 \pm 0.00$ & $0.48 \pm 0.02$ \\ \hline %& $0.60 \pm 0.11$ & $0.50 \pm 0.02$ & $0.58 \pm 0.01$ & $1.00 \pm 0.00$ & $0.00 \pm 0.00$ & $0.79 \pm 0.13$ & $0.99 \pm 0.02$ & $0.71 \pm 0.03$ & $1.00 \pm 0.00$ & $1.00 \pm 0.00$ \\ \hline
% iqr% iqr\\ \hline %
\textbf{IQR} \cite{buhlmann2014cam}       & $0.58 \pm 0.09$ & $0.46 \pm 0.01$ & $0.49 \pm 0.01$ & $0.54 \pm 0.01$ & $0.65 \pm 0.02$ & $0.52 \pm 0.02$ & $0.47 \pm 0.01$ & $0.51 \pm 0.01$ & $0.45 \pm 0.01$ & $0.52 \pm 0.01$ \\ \hline %& $0.49 \pm 0.00$ & $0.50 \pm 0.00$ & $0.52 \pm 0.01$ & $0.47 \pm 0.01$ & $0.52 \pm 0.01$ & $0.46 \pm 0.01$ & $0.04 \pm 0.02$ & $0.39 \pm 0.00$ & $0.68 \pm 0.04$ & $0.54 \pm 0.01$ \\ \hline

\end{tabular}
\end{sc}}
\end{center}
\vspace{-4ex}
\end{table*}

\subsection{Rank Correlation Coefficients}
Note that the bound in Theorem~\ref{thm:residual_bound} directly implies that the ranking dependence scores have global sensitivity $1$ (equal to the size of their ranges). To see this note that we can consider an adversarial situation in which the rank of every element of the residual $\rr_Y'$ changes when the training set is altered in one element (as all the residual elements may change). This means that the Laplace mechanism cannot guarantee useful privacy.

Instead, note that both ranking scores may still have reasonably bounded local sensitivity. Specifically, if we consider the list of sorted residuals, it may be that there are large gaps between neighboring residuals. If this is the case then changing the training set by one point may not change the residual rankings. Thus, the ranking scores are in some sense stable to changes in the training set (for certain sets).

\begin{define}
We call a function $f$ \textbf{$k$-stable on dataset $\Dc$} if modifying any $k$ elements in $\Dc$ does not change the value of $f$. Specifically, $f(\Dc) = f(\Dc^*)$ for all $\Dc^*$ such that $\Dc$ can be transformed into $\Dc^*$ with a minimum of $k$ element substitutions. We say $f$ is \textbf{unstable} on $\Dc$ if it is not even $1$-stable on $\Dc$. 
The \textbf{distance to instability} of a dataset $\Dc$ w.r.t. a function $f$ is the number of elements that must be changed to reach an unstable dataset.
\end{define}

With these definitions, we will use a modification of the Propose-Test-Release framework that makes use of this stability as described in Algorithm 13 in \citet{dwork2013algorithmic}.

% \begin{algorithm}[H]                      % enter the algorithm environment
% \caption{ Stable Propose-Test-Release \cite{dwork2013algorithmic}  }          % give the algorithm a caption
% \label{alg:rank_ptr}                           % and a label for \ref{} commands later in the document
% \begin{algorithmic}[1]                    % enter the algorithmic environment
% %	\STATE \textbf{Input:} $\Vc$; $\Lambda \subseteq \mathbb{R}^d$; $T$; $(\epsilon,\delta)$; $\sigma^2_{\Vc,0}$; $\gamma_T$
% %	\STATE $\mu_{\Vc,0} = 0$
% 	\STATE {\textbf{Input:} residuals $\rr_Y' = [r_{1,Y}',\ldots,r_{m,Y}']^\top$, privacy $\epsilon,\delta > 0$}
% 	\STATE $d = $ distance to training instability of of $\rr_Y'$ %\lfloor \log_{2} \mbox{IQR}(\Xb)  \rfloor$
% 	\STATE $p = d + z$, where $z \sim \mbox{Lap}(0,\frac{1}{\epsilon})$
% 	\IF{ $p > \frac{log(1/\delta)}{\epsilon}$ }
% 	\RETURN $s(\x',\rr_Y')$
% 	\ELSE
% 	\RETURN $\perp$
% 	\ENDIF
% \end{algorithmic}
% \end{algorithm}

% \begin{thm}
% Algorithm 13 \cite{dwork2013algorithmic} is $(\epsilon,\delta)$-differentially private.
% \end{thm}

\begin{thm} \cite{dwork2013algorithmic}
Algorithm 13 \cite{dwork2013algorithmic} is $(\epsilon,\delta)$-differentially private. Further, for all $\beta > 0$ if $s(\x',\rr_Y')$ is $\frac{\log(1/\delta) + \log(1/\beta)}{\epsilon}$-stable on $\rr_Y'$, then Algorithm 13 releases $s(\x',\rr_Y')$ w.p. at least $1-\beta$.
\end{thm}
% and to have the following utility guarantee.  Overall, this gives us a way to use the training set local sensitivity in the ranking scores.

A lower bound on the distance to instability $d$ is easily given by noting that $s(\x',\rr_Y')$ always outputs the same result as long as none of the ranks of $\rr_Y'$ change. Let $\gamma$ be the smallest absolute distance between any two ranks. Then a lower bound on $d$ is, 
%\begin{align}
$d > \lfloor n \gamma \lambda^{3/2}/16 \rfloor$. % \nonumber %\min\{ \lambda, \lambda^{3/4} \} \sqrt{n} }{ 4 \sqrt{2(1+\sqrt{2})} } \Bigg\rfloor. \nonumber
%\end{align}
%$\lfloor \gamma \lambda n / (1 + 1/\sqrt{\lambda}) \rfloor$. 
% \begin{align}
% \max_{k \in \{1,\ldots, n\}} k \;\; \mbox{ s.t. } \;\; \frac{\gamma}{2} > k \Big(1  + \frac{1}{\sqrt{\lambda}}\Big)\frac{2}{\lambda n} \nonumber
% \end{align}
This is the largest number of training points that may change so that the closest ranks moving towards each other do not overlap (given that they change by at most the amount in eq.~\ref{eq:residual_bound}). This lower-bound is sufficient to use Algorithm 13 \cite{dwork2013algorithmic} to privatize the ranking dependence scores.

\subsection{HSIC Score}
\begin{thm}
For $m \geq 2$, with kernels $k,l \leq 1$ where $l$ is $L_l$-Lipschitz, the HSIC score has a training set sensitivity as follows,
\begin{align}
\Big| \widehat{\textrm{HSIC}}_{k,l}(\x',\rr_Y') - \widehat{\mbox{HSIC}}_{k,l}(\x',\rbt_Y') \Big| \leq R \frac{32 L_l \sqrt{m}}{n} \nonumber
% \Big(1 + \frac{1}{\sqrt{\lambda}} \Big) \frac{64 L_l \sqrt{m} }{\lambda n} \nonumber
\end{align}
where $R =  \frac{8}{\lambda^{3/2}}$. %$ 2\sqrt{2(1 + \sqrt{2})}\Big/\min\{\lambda,\lambda^{3/4}\}$.
\end{thm}

The proof follows directly from Theorem~\ref{thm:residual_bound} and Lemma 16 in \citet{mooij2014distinguishing}. Thus, the Laplace mechanism gives us $(\epsilon,0)$-differential privacy and Theorem~\ref{thm:hsic_util} gives us our utility guarantee.

% gives us an immediate bound for the HSIC score:

% \begin{lem}
% \cite{mooij2014distinguishing} For $m \geq 2$, for all $\x' \in \mathcal{R}^m$, for all $\rr_Y', \rbt_Y' \in \mathcal{R}^m$, for kernels $k,l \leq 1$, where $l$ is also Lipschtiz-continuous, with constant $L_{l}$, we have that
% \begin{align}
% \Big| \widehat{\mbox{HSIC}}_{k,l}(\x',\rr_Y') - \widehat{\mbox{HSIC}}_{k,l}(\x',\rbt_Y') \Big| \leq \frac{32 L_l }{\sqrt{m}} \| \rr' - \rbt_Y' \|_2 \nonumber
% \end{align}
% \end{lem}
% and similarly for $\y',\rr_X'$. Using the bound in eq.~(\ref{eq:residual_bound}) we proves the following theorem:

% \begin{thm}
% The HSIC score has a training set sensitivity as follows,
% \begin{align}
% \Big| \widehat{\textrm{HSIC}}_{k,l}(\x',\rr_Y') - \widehat{\mbox{HSIC}}_{k,l}(\x',\rbt_Y') \Big| \leq \frac{64 L_l L_\ell C_k}{\lambda n} \nonumber
% \end{align}
% \end{thm}
% The same sensitivity holds for the residuals $\rr_X',\rbt_X'$ as well.

\subsection{IQR Score}
% Note that given a bound on how much individual residuals can change in Theorem~\ref{thm:residual_bound} it implies the following result
% \begin{thm}
% The IQR score has a training set global sensitivity bound of
% \begin{align}
% \Big| \mbox{IQR}(\rr_Y') - \mbox{IQR}(\rbt_Y') \Big| \leq \Big(1 + \frac{1}{\sqrt{\lambda}}\Big) \frac{4}{\lambda n}  \nonumber
% \end{align}
% \end{thm}
% \begin{proof}
% This bound follows from the fact that if each residual is only allowed to change by the amount in eq.~(\ref{eq:residual_bound}) then the first and third quartiles can grow or shrink by at most this amount.
% \end{proof}
% Therefore, we can apply the Laplace mechanism only to the IQR scores computed on the residuals $\rr_Y',\rr_X'$ and achieve utility as given in Theorem~\ref{thm:hsic_util}, where $\Delta$ is the above sensitivity.
%Note that the log-IQR score does not have training set global sensitivity
Similar to the test set privacy section we will use propose-test-release to give a useful, private IQR score. In fact, we will use IQR algorithm almost identically, except that we will define $A_j$ as the number of training points required to move the IQR out of an interval. Note that a lower bound on $A_j$ is simply the number of points required to move every input less than the median to the left and every input larger than the median to the right (or the reverse of these), using the bound on $\rr$ in eq.~(\ref{eq:residual_bound}). The aforementioned privacy and utility results of the IQR propose-test-release framework apply here. The only difference is we just need to add noise to the IQR scores computed on the residuals, which implies $(6\epsilon,2\delta)$-privacy and that the results of Theorem~\ref{thm:iqr} can be tightened.% The above results demonstrate that we can perserve privacy end-to-end for Algorithm \ref{alg:anm_ci}.

\section{Results}
\label{sec:results}
%!TEX root=pci.tex

% \begin{figure*}[t!]
% \begin{center}
% %\vspace{-2ex}
% \centerline{\includegraphics[width=\textwidth]{PCI_Utility2}}
% \vspace{-2ex}
% \caption{Privacy and utility for both HSIC and IQR scores.}
% %\vspace{-10ex}
% \label{figure.utility}
% \end{center}
% \end{figure*}

We test our methods for private release of causal inference statistics on a small subsets from the Cause-Effect Pairs Competition collection \citet{guyonchallenge2013}. Specifically, we randomly select $10$ of the largest $25$ datasets that have a causal direction either $X \rightarrow Y$ or $Y \rightarrow X$. We average over $10$ random 50/50 train/test splits of the data. Table~\ref{table.error} shows the non-private accuracy of the four dependence scores over these datasets. We show the probability of correct causal inference changes as these scores are made private w.r.t. the test set. Note that these scores are often complementary, with the ranking-based scores performing well on datasets in which HSIC does worse, and vice-versa.

%Specifically, as the ranking correlations scores and the HSIC/IQR scores perform well on different datasets, we evaluated them separately. We select $8$ random datasets ($4$ each) out of the largest $50$ that have a causal direction either $X \rightarrow Y$ or $Y \rightarrow X$, and for which each dependence score implies the correct causal direction. We select datasets in this way simply because our privacy methods are only as good as the scores they work to privatize. Indeed, these are the cases for which the privacy-utility trade-off is non-trivial. In this section we just look at test-set privacy preservation. We average over $5$ random 30/70 train/test splits of the data.

%We limit our analysis to methods that preserve test-set privacy (the analysis is nearly the same for the training-set privacy methods). We average over $5$ random 30/70 train/test splits of the data.

Figure~\ref{figure.utility} shows the effect of privatization on the dependence scores: HSIC and IQR. Note that, for low $\epsilon$ (increased privacy), the probability of correct influence is lower as the amount of noise required blurs the true dependence scores. However, as $\epsilon$ increases, so does this probability, in some cases drastically. For the IQR score, recall that there is a probability that the algorithm returns null: $\perp$, if $R_j$ is less than a threshold controlled by $\delta$. We investigated this probability, by varying $\delta \in [10^{-5},10^{-2}]$ and sampling $10,000$ points from the appropriate Laplace distribution. We found that, for the IQR dataset in figure~\ref{figure.utility} \emph{every sample} did not move $R_j$ below the null threshold. Therefore, the probability of null is essentially $0$.

The three left-most plots in Figure~\ref{figure.utility_tr} demonstrate how $\lambda$, which has a large effect on the training set sensitivity (as described in eq.~\ref{eq:residual_bound}) affects the probability of correct inference. We perform this experiment for different settings of $\epsilon$, and each one produces a distinctive `hump' shape. This is because for small $\lambda$ the sensitivity bound (\ref{eq:residual_bound}) is too large to produce meaningful causal inference. Similarly, for large $\lambda$ the kernelized regression algorithm (\ref{eq:kernel_ERM}) is overly-regularized, which produces a poor regressor and poor dependence scores. Only when $\lambda$ is within a certain range do we balance the size of the sensitivity bound with the size of the regularization. This range grows larger as $\epsilon$ increases as the privacy setting becomes less strict (requiring less noise). The right-most plot shows the correct inference probability using the best $\lambda$ for a range of $\epsilon \in [0.1,10]$. With proper selection of $\lambda$ we can achieve high-quality causal inference that maintains privacy w.r.t. the training set.

%The last row of the figure shows the probability that the IQR score is released, as given by the propose-test-release framework. Note that as, $\delta$ increases, the release probability increases as it directly controls the release threshold. If $\delta$ is small enough, increasing $\epsilon$ actually reduces the release probability as it is less likely that the Laplace noise in line 7 of Algorithm~\ref{alg:iqr_ptr} pushes $R_j$ over the release threshold.

 %Specifically, we compare the above methods that preserve test-set privacy (the analysis is nearly the same for the training-set privacy methods). As all privacy-preserving mechanisms are only as good as the algorithms they aim to privatize we look at datasets for which the dependence scores perform well. Indeed, these are the cases for which the privacy-utility trade-off is non-trivial. We analyze this trade-off for both HSIC and IQR scores and give an intuition for preferring one score over the other.

%Figure~\ref{figure.utility} shows the results for both scores.

\section{Conclusion}
\label{sec:conclusion}
%!TEX root=pci.tex
We have presented, to the best of our knowledge, the first work towards differentially private causal inference. 
There are numerous directions of future work including privatizing other causal inference frameworks (e.g. IGCI \cite{janzing2012information}), analyzing that ANM algorithm without train/test splits, as well as other dependence scores. As there is significant overlap in the applications of causal inference and private learning we believe this work constitutes an important step towards making causal inference practical.

 \section*{Acknowledgments} 
 KQW and MJK are supported by NSF grants IIA-1355406, IIS-1149882, EFRI-1137211. We thank the anonymous reviewers for their useful comments.

\bibliographystyle{icml2015}
\begin{small}
\bibliography{pci}
\end{small}

\end{document}